\documentclass{article}

\PassOptionsToPackage{numbers,compress}{natbib}

\usepackage[main, final]{neurips_2026}

\usepackage[utf8]{inputenc} 
\usepackage[T1]{fontenc}    
\usepackage{hyperref}       
\usepackage{url}            
\usepackage{booktabs}       
\usepackage{amsfonts}       
\usepackage{nicefrac}       
\usepackage{microtype}      
\usepackage{xcolor}         
\usepackage{amsmath}
\usepackage{amssymb}
\usepackage{graphicx}
\usepackage{amsthm}
\newtheorem{theorem}{Theorem}[section]
\newtheorem{proposition}[theorem]{Proposition}

\newtheorem{definition}[theorem]{Definition}
\newtheorem{assumption}[theorem]{Assumption}
\newtheorem{remark}[theorem]{Remark}
\usepackage{amsmath}
\usepackage{cancel}
\usepackage{multirow}
\title{Why Latent Actions Fail, and How to Prevent It}

\author{
  Jung Min Lee$^{1}$ \quad
  Taehyun Cho$^{1,2,3}$ \quad
  Li Zhao$^{4}$ \quad
  Jungwoo Lee$^{1,5}$ \\
  $^{1}$Seoul National University
  \qquad
  $^{2}$Vector Institute \\
  \qquad
  $^{3}$Fields Institute
  \qquad
  $^{4}$Microsoft Research Asia
  \qquad
  $^{5}$Hodoo AI Labs \\
  \texttt{\{jmleeluck,talium,junglee\}@snu.ac.kr}
  \qquad
  \texttt{lizo@microsoft.com}
}
\begin{document}

\maketitle
\begin{abstract}
Latent action models (LAMs) aim to learn manipulation priors from unlabeled videos by compressing frame-to-frame changes.
The frames of in-the-wild videos, however, contain not only the agent's own state but \textit{exogenous state} such as background clutter.
Since the exogenous state introduces changes unrelated to actions, it hinders reliable latent action learning.
This paper investigates this problem analytically by extending a linear LAM framework to explicitly model exogenous state.
Our analysis reveals two insights: (1) minimizing the standard reconstruction objective produces latent actions that encode exogenous information from future observation; and (2) learning in a representation space that focuses on endogenous components is a key to mitigating the interference of noise.
We further show that previously proposed auxiliary objectives, such as action-supervision, provably encourage latent actions to be consistent across exogenous states.
These findings are validated through experiments on both linear and nonlinear LAMs, providing a unified theoretical analysis of how exogenous state hinders latent action learning and why common remedies work.
\end{abstract}


\section{Introduction}\label{sec:intro}
Latent action models (LAMs) have emerged as a promising approach for learning action-relevant features from videos without action labels. 
This is particularly advantageous to real-world robotics, where collecting action-labeled datasets at scale is challenging~\cite{lfvsurvey}. 
Obtaining action annotations typically requires extensive human-expert teleoperation, careful curation, and repeated calibration, which creates a major bottleneck for dataset scaling.
LAMs mitigate this limitationThere results support by learning from unlabeled videos and producing pseudo-action labels that supervise downstream models, such as vision-language-action models (VLA)~\cite{lapa, univla, villa, mvp_lam} and world models~\cite{genie, adaworld}. 

For pseudo-action supervision to be effective, latent actions must align with ground-truth actions even in the absence of action labels. 
This is nontrivial in in-the-wild videos, where observations often contain manipulation-irrelevant variation such as background motion and camera-induced changes. 
Prior works theoretically show that latent actions can spuriously encode such \textit{exogenous noise}~\cite{linear_lam, misra2024towards}.
For instance, latent actions trained on egocentric data tend to encode camera motion rather than physical change, misleading downstream models.
Therefore, to improve action alignment, reducing exogenous noise is necessary.
Existing approaches reduce it by introducing additional training objectives, such as a small amount of action (or optical-flow) supervision, language descriptions, or multi-view datasets~\cite{univla, mvp_lam, laom, laof}.
These methods augment the standard LAM objective with tractable auxiliary objectives to mitigate exogenous noise.
However, why exogenous noise degrades latent actions and how auxiliary objectives mitigate it have not been analyzed in a unified manner.

In this work, we provide a unified theoretical analysis of how exogenous variation degrades latent action learning and why auxiliary training objectives mitigate it, building on the linear LAM framework. 
We first extend linear LAM to Exogenous Block Markov Decision Process (Ex-BMDP) to explicitly model a source of exogenous noise.
Under this formulation, we show that training a LAM with a reconstruction-based objective causes exogenous information from the future frame to leak into the latent action, reducing the capacity available for the action.
We further analyze how the strength of exogenous noise behaves and explain why previously proposed auxiliary objectives improve action alignment.
Our contributions are as follows.
\begin{enumerate}
    \item We formulate and analyze latent action learning in Ex-BMDP, providing a framework that explicitly models the source of exogenous noise. 
    Within this framework, we show that latent actions learned by reconstruction leak exogenous information from future frames, and that the energy of the resulting noise is governed by the sensitivity of the observation to exogenous change.
    \item We provide a theoretical explanation for why previously proposed auxiliary objectives are effective under the Ex-BMDP framework, showing that they mitigate exogenous noise by encouraging consistency of latent actions across exogenous states.
    \item We validate our analysis in controlled linear simulations and nonlinear LAMs.
    Additional experiments on real-world robot data and downstream control also support our analysis.
    These results support the applicability of our analysis beyond the linear setting.
\end{enumerate}

\begin{figure*}[!t]
    \centering
    \includegraphics[width=\linewidth]{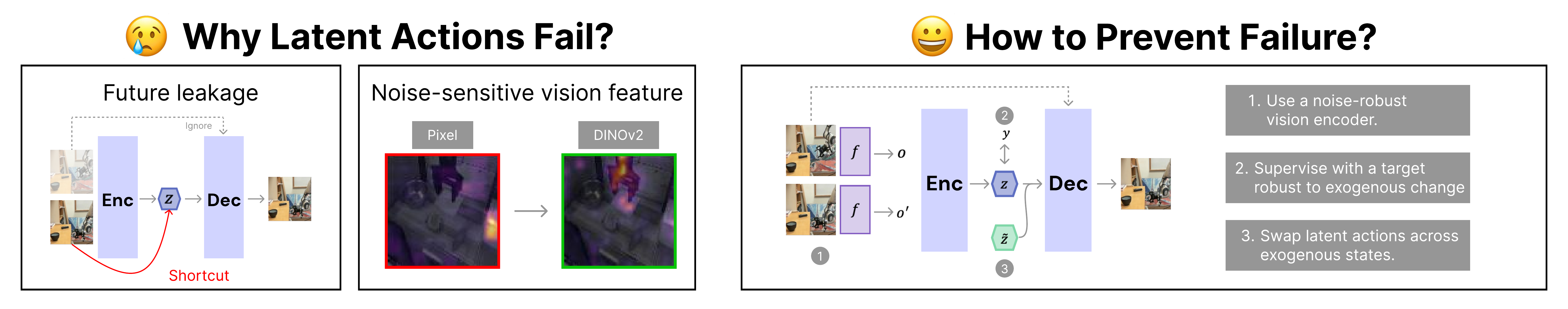}
    \caption{\textbf{Failure modes and remedies for latent action learning.}
    \textit{(Left)} Future leakage and noise-sensitive visual features hinder the learning of action-centric latent representations.
    \textit{(Right)} Our theoretical analysis explains why three remedies work: a noise-robust vision encoder, supervision with a noise-robust target $y$, and latent action swapping across exogenous states.}
    \label{fig:placeholder}
\end{figure*}

\section{Related Work}
\paragraph{Latent Action Learning from Videos.}
Recent works have made substantial progress in learning latent actions from unlabeled videos. 
One line of research learns discrete latent actions using vector-quantized variational autoencoders (VQ-VAE)~\cite{vqvae}. 
These discretized latent actions have been used to supervise VLA models~\cite{lapa, univla, mvp_lam} and world models~\cite{genie} even when the ground-truth actions are not provided.
In particular, latent actions have shown strong promise in robot manipulation, including cross-embodiment transfer~\cite{lapa} and the ability to scale VLA training with human videos in addition to robot datasets~\cite{univla, igor}. 
While these approaches offer effective ways to incorporate latent actions into downstream embodied AI systems, they offer a limited understanding of when and why latent actions align with ground-truth actions.

\paragraph{Exogenous Noise in Latent Action Learning}
Exogenous noise causes frame transitions, but is independent of the agent's state and action.
For instance, the motion of other agents or background movements are representative examples of it.
Prior works show that such noise hinders reliable latent action learning~\cite{linear_lam, misra2024towards, laom}.
\citet{linear_lam} show that LAMs may prioritize exogenous noise over action-induced changes when noise has larger variance.
However, their formulation under controlled Markov process (CMP) does not explicitly represent exogenous noise and it is introduced only in the linear analysis as an additive residual in the observation transition.
This leaves the open question about how exogenous noise should be formulated and analyzed beyond the additive linear setting.

To mitigate exogenous interference in practice, LAOM~\cite{laom} and LAOF~\cite{laof} supervise latent action learning with a small amount of action labels and optical flow, respectively.
UniVLA~\cite{univla} uses language descriptions to encode task-centric latent actions, while MVP-LAM~\cite{mvp_lam} leverages multi-view datasets to encode consistent transitions across viewpoints, achieving higher mutual information with actions.
Despite promising empirical results, the theoretical basis for why such auxiliary objectives alleviate exogenous interference remains an open question.

\section{Setup}
In this section, we first discuss the Exogenous Block Markov decision process (Ex-BMDP) proposed by \citet{exogeneous_idm}.
Ex-BMDP allows us to describe the general latent action learning in various video datasets.
Then, we explain the standard latent action learning scenario in Ex-BMDP and extend the linear LAM setting to Ex-BMDP.

\subsection{Exogenous Block Markov Decision Process}
For a given set $\mathcal{U}$, we use $\Delta(\mathcal{U})$ to denote the set of all possible probability distributions over $\mathcal{U}$. 
We first consider the Block Markov Decision Process (BMDP) system.
BMDP is specified by 
\begin{equation}
    \mathcal{M} = (\mathcal{O}, \mathcal{X}, \mathcal{A}, T, q, R, H, \mu)
\end{equation}
where $\mathcal{O}$ is a set of observations, $\mathcal{X}$ is a set of latent states, $\mathcal{A}$ is a finite set of actions. 
We first start from initial distribution $\mu \in \Delta(\mathcal{X})$, and the agent interacts with the environment by repeatedly generating $H$-step trajectories $(x_1, o_1, a_1, r_1, \dots, x_H, o_H, a_H, r_H)$ where the agent first receives observation $o_t \sim q(\cdot\mid x_t)$ from emission function $q:\mathcal{X} \rightarrow \Delta(\mathcal{O})$ and chooses actions $a_t \sim \pi(\cdot|o_t)$ using policy $\pi:\mathcal{O} \rightarrow \Delta(\mathcal{A})$. 
Then, the latent state evolves $x_{t+1} \sim T(\cdot\mid x_t, a_t)$ and the agent receives a reward $r_t = R(x_t, a_t)$.

Ex-BMDP is an extension of BMDP by assuming that $\mathcal{X} = \mathcal{S} \times \Xi$ where $\mathcal{S}$ is the set of (minimal) \textit{endogenous states}, which is the (minimal) information that is necessary to control the agents and fully discards unnecessary information \cite{laom, minimal_endogenous}, and $\Xi$ is a set of \textit{exogenous states}. 
For $x \in \mathcal{X}$, Ex-BMDP further assumes that the initial distribution and transition functions are decoupled, that is
\begin{equation}
    \mu(x) = \mu_s(s)\mu_\xi(\xi),\qquad T(x'|x, a) = T_s(s'|s,a)T_\xi(\xi'|\xi)
\end{equation}
where $\mu_s\in \Delta(\mathcal{S})$ and $T_s: \mathcal{S} \times \mathcal{A} \rightarrow \Delta(\mathcal{S})$ are initial distribution and transition function over state space, while $\mu_{\xi} \in \Delta(\Xi)$ and $T_\xi:\Xi \rightarrow \Delta(\Xi)$ are those over the exogenous state space.

\subsection{Standard LAM training}
\paragraph{Ex-BMDP in latent action learning.}
The standard latent action learning scenario can be expressed under Ex-BMDP. 
Assume that the image is generated by a deterministic rendering function $I = g(s, \xi)$ where $g: \mathcal{S}\times \Xi \rightarrow \mathbb{R}^{C\times H \times W}$.
To model the common LAM setting with a vision encoder, we further introduce an encoder $f(\cdot): \mathbb{R}^{C \times H \times W} \rightarrow \mathcal{O}$,
and define the observation as 
\begin{equation}
    o = f(I) = f(g(s, \xi)) \equiv h(s, \xi).
\end{equation}
Here, $h:\mathcal{S}\times \Xi \rightarrow \mathcal{O}$ denotes the deterministic emission map from $(s, \xi)$ to $\mathcal{O}$. 
This formulation also includes the pixel-space setting as a special case, where $f$ is the identity map and $\mathcal{O} = \mathbb{R}^{C \times H \times W}$.
We additionally assume the data-collecting policy is an \emph{endogenous policy}, i.e., $\pi(a|o) = \pi(a|s)$~\cite{misra2024towards, minimal_endogenous}.
This ensures that the action $a$ does not depend on exogenous variation given the endogenous state.

\paragraph{Latent action learning.}
Given a dataset $\mathcal{D}$ of consecutive observations, a latent action model learns a latent action from a pair of observations and predicts the next observation from the current observation and the inferred latent action. Concretely, the inverse dynamics model $\psi_{\mathrm{IDM}}$ infers a latent action $z$ from $(o, o')$, and the forward dynamics model $\psi_{\mathrm{FDM}}$ reconstructs $o'$ from $(o, z)$
\begin{equation}
    z = \psi_\mathrm{IDM}(o, o'), \quad
    \hat{o}' = \psi_\mathrm{FDM}(o, z),
\end{equation}
with the standard LAM objective
\begin{equation}
    \mathcal{L}_\mathrm{LAM}
    =
    \min_{\psi_\mathrm{IDM}, \psi_\mathrm{FDM}}
    \mathbb{E}_{\mathcal{D}}
    \left[
        \lVert o' - \hat{o}' \rVert_2^2
    \right].
    \label{eqn:lam_obj}
\end{equation}
Both models are trained in an end-to-end manner with Equation \eqref{eqn:lam_obj}.

\paragraph{Goal.}
Our goal is to learn latent actions from reward- and action-free data that are highly informative about ground-truth actions.
Following \citet{linear_lam} and \citet{mvp_lam}, we formalize this desideratum as maximizing the mutual information between the latent action $z$ and the action $a$:
\begin{equation}
\max_{z\in \mathbb{R}^{d_z}}\;\mathcal{I}(z;\,a).
\label{eqn:mi}
\end{equation}

\begin{figure}[!t]
    \centering
    \includegraphics[width=\linewidth]{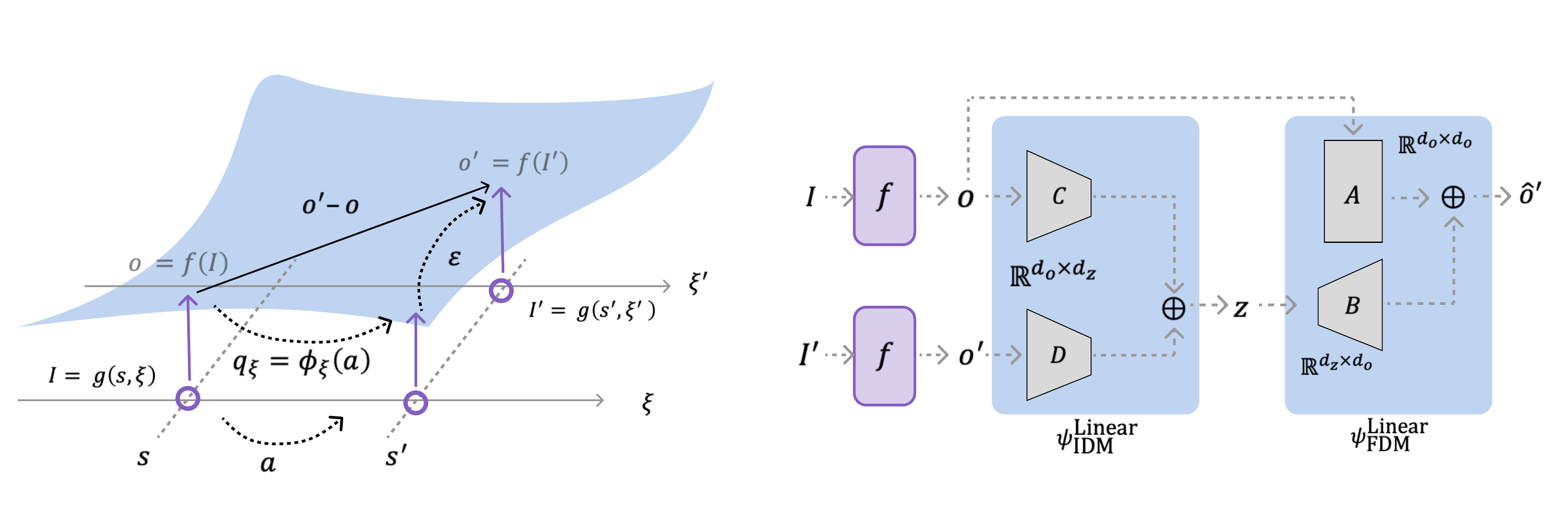}
    \caption{\textbf{Overview of Linear LAM.} \textit{(Left)}: Observation transitions contain a state-driven component $  q_\xi = \phi_\xi(a)$ and a $\xi$-driven exogenous component $\varepsilon$. 
    \textit{(Right)}: Architecture of linear LAM with a pretrained vision encoder $f$.}
    \label{fig:overview}
\end{figure}

\subsection{Linear LAM setting}
The linear LAM framework provides theoretical simplicity to analyze latent actions~\cite{linear_lam}.
Consider observations $o \in \mathbb{R}^{d_o}$ and the action $a \in \mathbb{R}^{d_a}$ which has lower dimensionality ($d_a \ll d_o$).
For theoretical simplicity, linear LAM assumes that the forward dynamics are additive in observation space, consisting of the observation $o$, a controllable change $q_\xi = \phi_\xi(a)$ where $\phi_\xi: \mathbb{R}^{d_a}\rightarrow \mathbb{R}^{d_o}$ is the action-effect function, and exogenous noise $\varepsilon$.
\begin{equation}
    o' = o + q_\xi +\varepsilon
    \label{eqn:linear_lam_dyn}
\end{equation}
Also, linear LAM assumes that $\psi_\mathrm{IDM}$ and $\psi_\mathrm{FDM}$ are implemented as linear maps
\begin{align}
    z &= \psi_\mathrm{IDM}^\mathrm{Linear}(o, o') = Co + Do', \\
    \hat{o}' &= \psi_\mathrm{FDM}^\mathrm{Linear}(o, z) = Ao + Bz,
\end{align}
where $z \in \mathbb{R}^{d_z}$ with $d_z \ll d_o$.
All parameters $A \in \mathbb{R}^{d_o\times d_o}$, $B \in \mathbb{R}^{d_o \times d_z}$, and $C, D \in \mathbb{R}^{d_z \times d_o}$ are learnable.
We let $\theta = (A,B,C,D)$ for brevity. 
Figure~\ref{fig:overview} (Right) overviews this setup.

\section{Theoretical Analysis and Experiments}\label{sec:theory}
We now study how the exogenous state affects latent action learning.
We focus on the regime where the energy of $\varepsilon$ is comparable to or exceeds that of $q_\xi$, the typical setting for latent action learning from in-the-wild videos.
We first show that the reconstruction-based objective causes exogenous information from the future frame to leak into latent actions (Section \ref{sec:future_leak}). 
We then identify the visual representation as the dominant factor controlling the energy of exogenous noise (Section \ref{sec:exo_energy}). 
Finally, we analyze two families of auxiliary objectives proposed in prior work and show that both enforce consistency of latent actions across exogenous states (Section \ref{sec:cross_exo} and \ref{sec:exo_robust}).

\paragraph{Linear LAM experiment.}
Before our theoretical analysis, we first describe the linear LAM experimental setup used to empirically validate the results throughout this section.
We first construct synthetic trajectories.
The state evolves as $s_{t+1}=s_t+a_t$, where $a_t\sim\mathcal{N}(0,I)$.
For the exogenous state, we consider $n_\xi$ number of exogenous states and they switch with probability $p_{\mathrm{switch}} \in [0, 1]$ by uniformly choosing from the other $n_\xi-1$ exogenous states.
Observations are generated by a $\xi$-dependent linear map $o=h(s,\xi)=H_{\xi}s$ where $H_\xi=H_0+\alpha R_\xi$ with random matrices $R_\xi$ and positive constant $\alpha \in \mathbb{R}_{\ge0}$.
To make $R_\xi$, we first initialize $n_{\xi}$ different random matrices and choose one of them for each $\xi$.
Note that $H_0$ is fixed across the exogenous state $\xi$.

To evaluate the action-alignment of latent actions, we train a linear probe \cite{linear_probe} with a simple mean squared error (MSE) objective and measure the normalized mean squared error (NMSE) of the linear probe.
This measures \textit{how much action-relevant information is included in latent actions} and it is widely used to evaluate the quality of latent action \cite{mvp_lam, laom}.
This evaluation metric serves as a proxy for Equation~\eqref{eqn:mi}, since for jointly Gaussian $z$ and $a$, $-\log \mathbb{E}[\|\hat{a}(z) - a\|^2]$ lower-bounds $\mathcal{I}(z;a)$ up to a constant, where $\hat{a}(z)$ is the least-squares estimate of $a$ given $z$.
Here, we freeze the LAM and evaluate the quality of latent actions using a linear map $\hat{a}(z) = Mz + b$ where $M \in \mathbb{R}^{d_a \times d_z}, b \in \mathbb{R}^{d_a}$ are trainable.
For each configuration, we report the mean and standard deviation across 4 seeds.
We provide additional details of linear LAM experiments in Appendix \ref{sec:sim_detail}.

\subsection{Additive dynamics in Ex-BMDP}
For theoretical tractability, we introduce the following assumptions on the Ex-BMDP, under which additive dynamics in Equation~\ref{eqn:linear_lam_dyn} are naturally derived.
\begin{assumption}\label{assmp:exbmdp_to_linear_lam}
Suppose we have the dataset $\mathcal{D}$ generated by an Ex-BMDP with deterministic emission $h: \mathcal{S}\times \Xi \rightarrow \mathcal{O}$ and deterministic endogenous transition $s' = T_s(s, a)$.
We assume that for all $s \in \mathcal{S}$, $\xi \in \Xi$, and $a \in \mathcal{A}$, there exists a function $\phi_\xi: \mathcal{A} \to \mathcal{O}$ such that
\begin{equation}\label{eqn:q_c}
\phi_\xi(a) = h(s', \xi) - h(s, \xi), \quad \text{where } s' = T_s(s, a).
\end{equation}
The exogenous noise $\varepsilon$ arises from variation of the exogenous state at a fixed endogenous state:
\begin{equation}\label{eqn:noise}
\varepsilon = h(s', \xi') - h(s', \xi).
\end{equation}
\end{assumption}

\paragraph{Justification of Assumption~\ref{assmp:exbmdp_to_linear_lam}.}
The key idea is to decompose the observation change into two components according to its source:
changes caused by the agent's action are captured through the endogenous transition under a fixed exogenous state (Equation~\eqref{eqn:q_c}), while changes unrelated to the action are attributed to exogenous state variation under a fixed endogenous state (Equation~\eqref{eqn:noise}).
Since we prefer additive dynamics for theoretical simplicity, we express both components as simple differences in observation space.
Although this decomposition is derived under the linear setting, the underlying principle---that action-relevant and action-irrelevant changes are observed through distinct state components---is not specific to linearity and is expected to hold more broadly.

With Assumption \ref{assmp:exbmdp_to_linear_lam}, we have
\begin{align*}
    \phi_\xi(a) + \varepsilon &= \cancel{h(s', \xi)} - h(s, \xi) + h(s',\xi') - \cancel{h(s', \xi)}\\ 
    &= h(s', \xi') - h(s, \xi)\\
    &= o' - o
\end{align*}
Therefore, linear LAM is induced by Ex-BMDP under Assumption \ref{assmp:exbmdp_to_linear_lam}.
Figure~\ref{fig:overview} (Left) summarizes this decomposition.

\subsection{Analysis 1: Future exogenous state leaks into latent action}\label{sec:future_leak}
A central question is what information the latent action $z$ encodes while minimizing $\mathcal{L}_\mathrm{LAM}$.
Previous work shows that $z$ captures both controllable changes and exogenous noise~\cite{linear_lam}.
Ideally, $z$ should encode only the action-induced change between $o$ and $o^\prime$, and any factor irrelevant to the action reduces $z$'s capacity for action-relevant information.
In this section, we revisit this issue under the Ex-BMDP framework.
The key insight is that $\psi_\mathrm{FDM}$ must reconstruct $o^\prime$ from $(o, z)$ while $\xi^\prime$ is observable only through $o^\prime$, so any $\xi^\prime$-dependent information must be carried by $z$.
We formalize this leakage below.

\begin{proposition}\label{prop:context_affect}
Suppose exogenous noise $\varepsilon$ is non-degenerate in the observation space, i.e., its covariance with respect to $\xi'$ is positive definite. 
Then, no global minimizer of $\mathcal{L}_\mathrm{LAM}$ in linear LAM admits a latent action that is $\xi'$-independent.
\end{proposition}

\begin{figure}[!t]
    \centering
    \includegraphics[width=\linewidth]{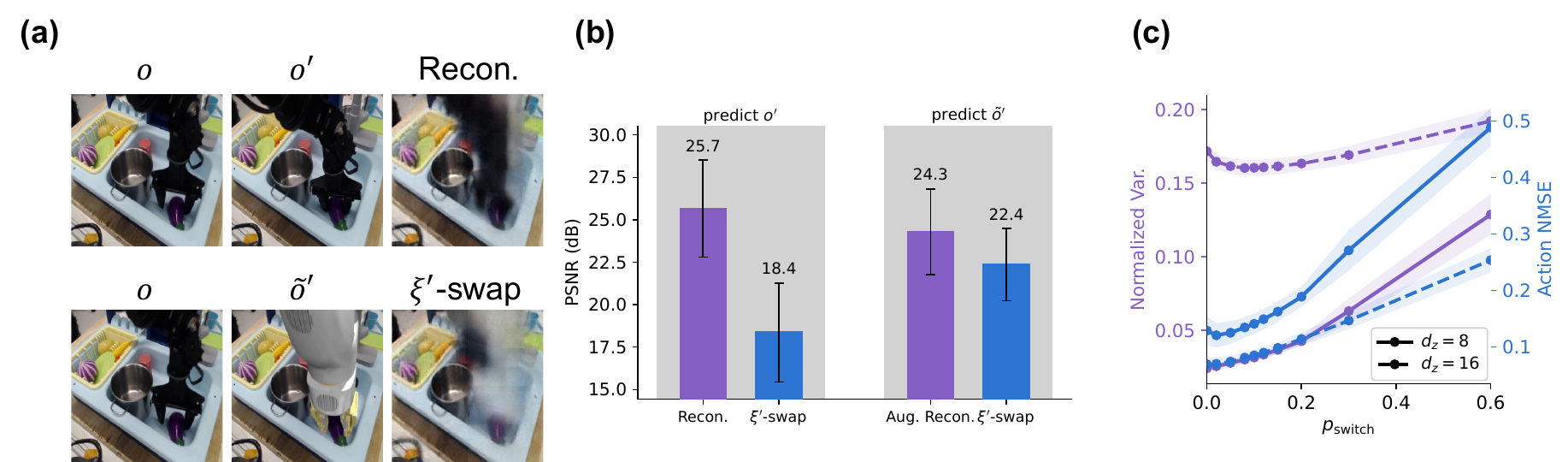}
    \caption{\textbf{Future exogenous state leaks into latent actions.}
    \textbf{(a)} Predicted next observation under the \texttt{Recon.} and $\xi'$-\texttt{swap}. 
    \textbf{(b)} PSNR of three settings. 
    $\xi^\prime$-\texttt{swap} drops sharply when targeting $o^\prime$ but matches \texttt{Aug. Recon.} when targeting $\tilde{o}^\prime$.
    \textbf{(c)} Normalized variance $\mathrm{Var}_{\xi^\prime}(z|s, \xi, a)/\|z\|^2$ and action NMSE as functions of $p_\mathrm{switch}$. Both metrics rise together as $p_\mathrm{switch}$ grows, indicating that $z$ encodes $\xi^\prime$ at the cost of action capacity.
    }
    \label{fig:analysis1}
\end{figure}

See Appendix~\ref{proof:context_affect} for the proof.
The non-degeneracy condition is natural in our setting, as large-scale in-the-wild videos typically contain diverse exogenous factors.
Under this condition, Proposition~\ref{prop:context_affect} formalizes the \emph{future leakage} empirically observed in recent LAMs~\cite{laof, vlajepa, como}, where $z$ absorbs information from the future frame instead of encoding the transition dynamics.
A $\xi^\prime$-dependent $z$ drives this leakage by allowing $\psi_\mathrm{FDM}$ to take a shortcut, copying the future frame rather than predicting it from the action.
This, in turn, leaves $z$ uninformative about what has changed due to the action.
Appendix \ref{sec:leakage_bound} further derives a lower bound of leakage under weight-decay scenario.

\paragraph{Experimental Settings.}
Motivated by~\citet{lawm}, we measure future leakage by changing the exogenous state of the next observation when encoding latent actions.
Ideally, $\psi_\mathrm{FDM}(o, z)$ should predict the next observation of $o$.
However, if LAM encodes the future observation, it would reconstruct that observation, even when it differs from the actual next observation of $o$.
To validate this in practice, we use LAPA~\cite{lapa} evaluated on Bridge V2~\cite{bridge} with paired transitions from OXE-AugE~\cite{oxeauge}, which augments the original trajectories in Open X-Embodiment~\cite{oxe} with simulated robot embodiments while preserving the underlying action.
Here, $o = h(s, \xi)$ and $\tilde{o} = h(s, \tilde{\xi})$ share the same state $s$ but differ in the exogenous state $\xi\neq\tilde{\xi}$ (different embodiments).
We compare three cases:
(i) \texttt{Recon.}: $\psi_\mathrm{FDM}(o, z)$ with $z = \psi_\mathrm{IDM}(o, o^\prime)$,
(ii) \texttt{Aug.~Recon.}: $\psi_\mathrm{FDM}(\tilde{o}, \tilde{z})$ with $\tilde{z} = \psi_\mathrm{IDM}(\tilde{o}, \tilde{o}^\prime)$,
and (iii) $\xi^\prime$-\texttt{swap}: $\psi_\mathrm{FDM}(o, z_\mathrm{cross})$ with $z_\mathrm{cross} = \psi_\mathrm{IDM}(o, \tilde{o}^\prime)$.
If $z_\mathrm{cross}$ did not encode $\tilde{\xi}^\prime$ and $\psi_\mathrm{FDM}$ did not rely on the shortcut through $z$, $\psi_\mathrm{FDM}(o, z_\mathrm{cross})$ would predict $o^\prime$, not $\tilde{o}^\prime$.

\paragraph{Results.}
As Figure~\ref{fig:analysis1}(a) shows, $\xi'$-\texttt{swap} outputs a frame resembling $\tilde{o}^\prime$ rather than $o^\prime$, despite being conditioned on $o$.
Quantitatively, as shown in Figure~\ref{fig:analysis1}(b), $\xi^\prime$-\texttt{swap} drops $7.3$ dB from \texttt{Recon.} in predicting $o^\prime$ but only $1.9$ dB from \texttt{Aug.~Recon.} in predicting $\tilde{o}^\prime$.
The $1.9$ dB gap is comparable to the $1.4$ dB drop from \texttt{Recon.} to \texttt{Aug.~Recon.}, showing that the degradation under $\xi^\prime$-\texttt{swap} is not attributable to distribution shift from the augmentation.
This confirms that LAM encodes the next exogenous state as part of $z$, which $\psi_\mathrm{FDM}$ then uses as a visual shortcut by copying the future frame into its prediction.
This in turn incentivizes $z$ to encode even more future information during training.
In Figure~\ref{fig:analysis1}(c), we train a linear LAM with varying $p_\mathrm{switch}$ and report the normalized variance $\mathrm{Var}_{\xi^\prime}(z|s, \xi, a) / \|z\|^2$ alongside action NMSE.
Both metrics grow with $p_\mathrm{switch}$, confirming that latent actions encode $\xi^\prime$ at the expense of action-relevant capacity.
This trend is consistent with theoretical result in Appendix \ref{sec:leakage_bound}.

\subsection{Analysis 2: Exogenous variation and sensitivity of vision encoder}\label{sec:exo_energy}
Given prior theoretical results that linear LAM behaves like principal component analysis (PCA)~\cite{linear_lam}, we now analyze \textit{which factors determine the energy of exogenous noise}.
\begin{remark}
For any fixed $s' \in \mathcal{S}$, the conditional second moment of $\varepsilon$ decomposes as
\begin{equation}
    \mathbb{E}\!\left[\|\varepsilon\|_2^2 \mid s'\right]
    =
    \underbrace{\mathbb{P}\!\left[\xi'\neq \xi \mid s'\right]}_{\text{frequency of }\xi \text{-changes}}\;
    \underbrace{\mathbb{E}\!\left[\|h(s',\xi') - h(s',\xi)\|_2^2 \mid s',\, \xi'\neq \xi\right]}_{\text{sensitivity of the vision encoder to } \xi \text{-change}}.
    \label{eqn:var_eps}
\end{equation}
\end{remark}

\begin{figure}[!t]
    \centering
    \includegraphics[width=\textwidth]{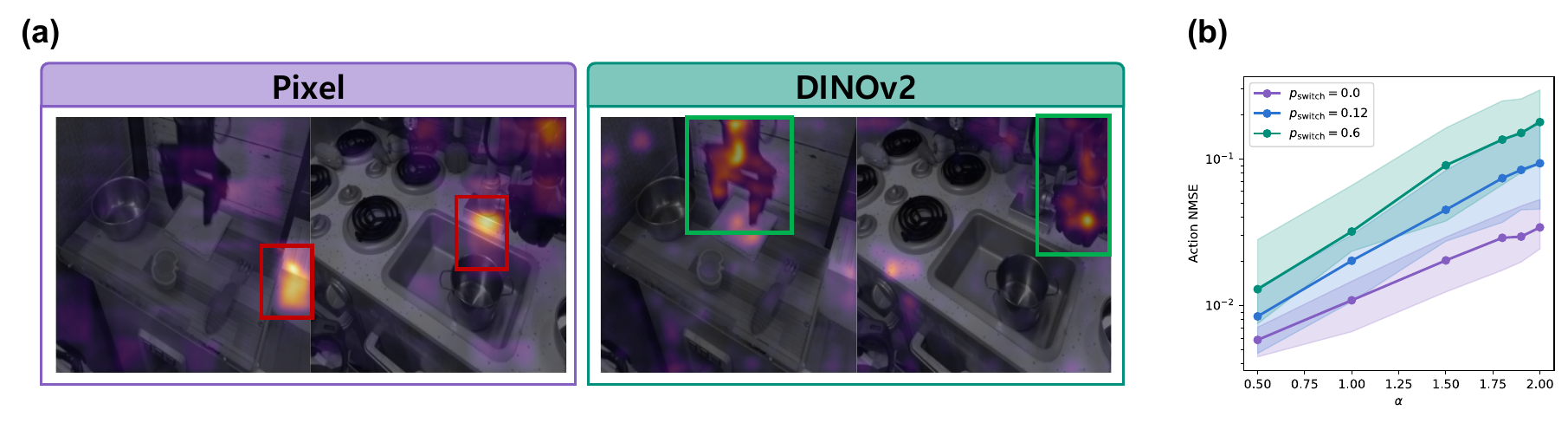}
    \caption{\textbf{Exogenous state sensitivity of observation impacts action alignment.} \textbf{(a)} Attention maps on Bridge V2 for LAMs trained with pixel (left) versus DINOv2 observations (right). 
    DINOv2 attends to manipulation-relevant regions (\textcolor{green}{green} boxes), while pixel-based LAMs often attend to background or manipulation-irrelevant factors (\textcolor{red}{red} boxes). 
    \textbf{(b)} In the linear LAM, action NMSE increases with $\alpha$ and with higher $p_\mathrm{switch}$.}
    \label{fig:analysis2}
\end{figure}

Equation~\eqref{eqn:var_eps} decomposes exogenous energy into two factors: (i) $\xi$-switching frequency and (ii) representation sensitivity to exogenous state at a fixed state.
Since the $\xi$-switching frequency $\mathbb{P}[\xi'\neq \xi | s']$ is defined by the environment itself, we focus on the representation sensitivity to exogenous factors.
Define the $\xi$-sensitivity $\delta_h$ of emission function $h$ as follows.
\begin{equation}
    \delta_h(s) = \sup_{\ \xi, \tilde{\xi}\in \Xi} \lVert h(s, \tilde{\xi}) - h(s, \xi) \rVert^2
\end{equation}
Then, the second moment of $\varepsilon$ satisfies
\begin{equation}
    \mathbb{E}[\lVert \varepsilon\rVert^2|s'] \le \mathbb{P}[\xi'\neq \xi|s']\delta_h (s')
    \label{eqn:noise_energy_sensitivity}
\end{equation}
As Equation~\eqref{eqn:noise_energy_sensitivity} indicates, reducing the $\xi$-sensitivity of the vision encoder reduces the energy of exogenous noise.
We empirically verify this in Figure~\ref{fig:analysis2}.
As Figure~\ref{fig:analysis2}(a) shows, DINOv2~\cite{dinov2} often provides an object-centric representation that is more robust to $\xi$, and therefore LAM trained on DINOv2 focuses on manipulation-relevant parts.
By contrast, LAM trained on raw images often fails to focus on manipulation-relevant features.
Accordingly, in linear LAM, Figure~\ref{fig:analysis2}(b) shows that action NMSE increases with the $\xi$-sensitivity of the vision encoder.
Here, we use $\alpha$ as a proxy for $\xi$-sensitivity.
These results highlight that the choice of vision encoder $f$ is important for latent action learning, motivating the use of representations with low $\xi$-sensitivity rather than raw pixels~\cite{univla, como, srl, stamo, mae}.

\subsection{Analysis 3: Cross-exogenous reconstruction learns CCA}\label{sec:cross_exo}

For theoretical simplicity, we assume access to synchronized transition tuples $(s,s')$ observed under two exogenous states $\xi$ and $\tilde{\xi}$.

\begin{definition}
Let $o = h(s,\xi)$ and $\tilde{o} = h(s,\tilde{\xi})$.
We define the cross-exogenous reconstruction objective as follows:
\begin{equation}
    \mathcal{L}_\mathrm{X\text{-}exo}(\theta) = \mathbb{E}[\lVert o' - \psi_\mathrm{FDM}(o, \tilde{z}) \rVert^2], \quad \tilde{z} = \psi_\mathrm{IDM}(\tilde{o}, \tilde{o}').
\end{equation}
\end{definition}
Under the linear LAM assumptions, the following proposition identifies the minimizer of this objective and clarifies its connection to canonical correlation analysis (CCA).
\begin{proposition}\label{prop:x_context}
    In the linear LAM setting, define $u = q_\xi +\varepsilon$ and $\tilde{u} = q_{\tilde{\xi}} + \tilde{\varepsilon}$ and suppose the observation is uncorrelated with both $u$ and $\tilde{u}$.
    Then, the cross-exogenous reconstruction objective reduces to
    \begin{equation}
        \mathcal{L}_\mathrm{X\text{-}exo}(\theta) = \mathbb{E}\left[\lVert u - BD\tilde{u}\rVert^2\right].
    \end{equation}
    When $u$ and $\tilde{u}$ are whitened, i.e., $\Sigma_{uu} = \Sigma_{\tilde{u}\tilde{u}} = I$, the optimal low-rank reconstruction is determined by the singular directions of $\Sigma_{u\tilde{u}}$, corresponding to the canonical directions of CCA.
\end{proposition}


\begin{figure}[!t]
    \centering
    \includegraphics[width=\textwidth]{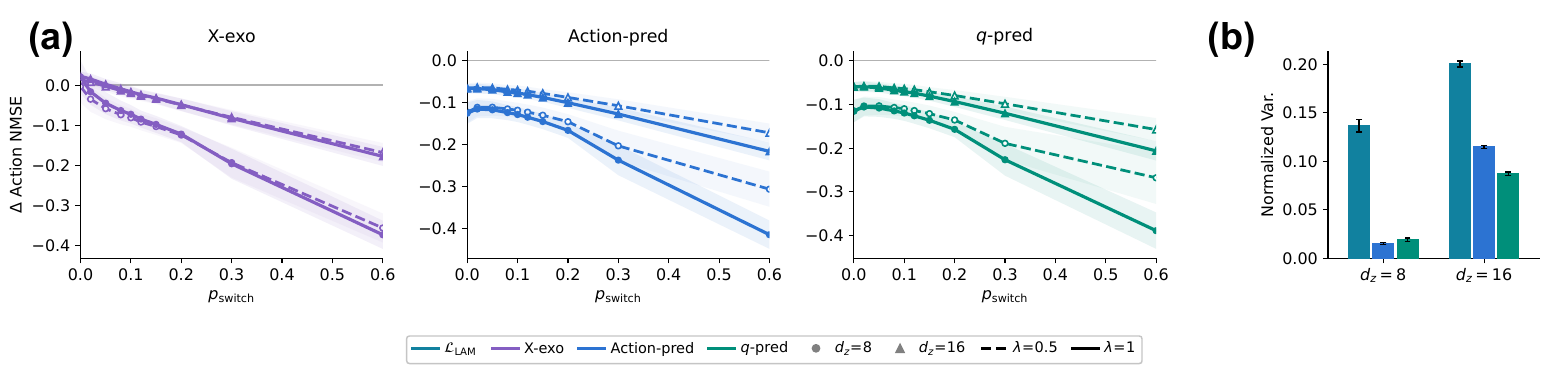}
    \caption{\textbf{Effect of $\mathcal{L}_\mathrm{X\text{-}exo}$ and $\mathcal{L}_{\xi\mathrm{\text{-}robust}}$.} \textbf{(a)} $\Delta$ action-validation NMSE relative to $\mathcal{L}_\mathrm{LAM}$ (negative values indicate improvement) for three auxiliary training objectives: (left) X-exo, (mid) Action-pred, and (right) $q$-pred. 
    \textbf{(b)} Normalized variance of latent action across exogenous states for different latent action dimensions $d_z$ (lower is better), measuring how consistent the latent actions are.
    Shaded regions in (a) and error bars in (b) show the standard deviation across 4 random seeds.
    }
    \label{fig:analysis3}
\end{figure}

See proof in Appendix \ref{proof:x_context} and further analysis in Appendix \ref{sec:x_context_no_assmp}.
Proposition~\ref{prop:x_context} shows that, under whitening, minimizing the cross-exogenous reconstruction objective is equivalent to CCA, which keeps transition directions that carry large energy under both $\xi$ and $\tilde{\xi}$ simultaneously. 
Transition components that have high energy under only a specific exogenous state are suppressed, while components that persist across exogenous states---which are more likely to be action-driven---are preserved.
This result is consistent with MVP-LAM~\cite{mvp_lam}, which treats viewpoint as the exogenous state and uses a cross-exogenous reconstruction objective on multi-view data to obtain latent actions well-aligned with ground-truth actions.
To investigate the effectiveness of $\mathcal{L}_\mathrm{X\text{-}exo}$, we train linear LAM with $\mathcal{L}_\mathrm{LAM} + \lambda \mathcal{L}_\mathrm{X\text{-}exo}$ and compare it to linear LAM trained only with $\mathcal{L}_\mathrm{LAM}$.
Figure~\ref{fig:analysis3} (a) (left) shows that $\mathcal{L}_\mathrm{X\text{-}exo}$ consistently improves linear probe action prediction.

\subsection{Analysis 4: Supervision of exogenous-robust target improves consistency}\label{sec:exo_robust}
Prior work has shown that action prediction improves the quality of latent actions and, in turn, downstream performance~\cite{linear_lam, laom}.
Since ground-truth action labels are rarely available at scale, a practical alternative is to predict optical flow obtained from pretrained estimators such as DPFlow~\cite{dpflow}, which requires no additional annotation.
In this section, we unify these prediction targets under a single formulation.
\begin{definition}
Let $y:\mathbb{R}^{d_o}\times\mathbb{R}^{d_o}\to\mathbb{R}^{d_y}$ be a target label of each observation pair.
We say that $y$ is $\eta$-robust to the exogenous state if, for any $(s, s') \in \mathcal{S}^2$ and any $\xi, \xi', \tilde{\xi}, \tilde{\xi}' \in \Xi$,
\begin{equation}
\left\|y(o, o') - y(\tilde{o}, \tilde{o}')\right\|^2 \le \eta,
\end{equation}
where $o = h(s, \xi)$, $o' = h(s', \xi')$, $\tilde{o} = h(s, \tilde{\xi})$, and $\tilde{o}' = h(s', \tilde{\xi}')$.
Typical examples include action labels and optical flow.
Given such a target, we define the exogenous-robust prediction objective
\begin{equation}
\mathcal{L}_{\xi\mathrm{\text{-}robust}}(\theta_\mathrm{IDM}, W; y)
= \mathbb{E}\!\left[\left\|y(o,o') - W\psi_\mathrm{IDM}(o,o')\right\|^2\right].
\end{equation}
where $\theta_\mathrm{IDM}$ denotes the parameters of $\psi_\mathrm{IDM}$ ($\theta_\mathrm{IDM} = (C,D)$ if linear LAM), and $W \in \mathbb{R}^{d_y \times d_z}$.
\end{definition}
The following proposition shows that using a supervision target $y$ with a small $\eta$ improves the consistency of latent actions across exogenous factors.

\begin{proposition}\label{prop:context_robust_prediction}
In the linear LAM setting, the exogenous-robust prediction objective with $\eta$-robust prediction target $y$ satisfies the following inequality:
\begin{equation}
    \mathbb{E}[\lVert W(\psi_\mathrm{IDM}(o, o') - \psi_\mathrm{IDM}(\tilde{o}, \tilde{o}')) \rVert^2] \le 6\mathcal{L}_{\xi\mathrm{\text{-}robust}}(\theta_\mathrm{IDM}, W;y) + 3\eta
\end{equation}
where $o = h(s,\xi)$ and $\tilde{o} = h(s,\tilde{\xi})$ for two distinct exogenous states $\xi\neq \tilde{\xi}$.
In particular, if $W$ has full column rank (which requires $d_y \ge d_z$), then the latent action $z = \psi_\mathrm{IDM}(o, o')$ becomes consistent across exogenous states as $\mathcal{L}_{\xi\text{-}\mathrm{robust}}$ and $\eta$ decrease.
\end{proposition}

See proof in Appendix \ref{proof:context_robust_prediction}.
Proposition~\ref{prop:context_robust_prediction} shows that predicting an $\eta$-robust target yields an explicit bound on latent action mismatch across the exogenous state.
Thus, minimizing $\mathcal{L}_{\xi\mathrm{\text{-}robust}}$ with a smaller $\eta$ encourages $\psi_\mathrm{IDM}$ to produce latent actions that are consistent across exogenous states.

To show how $\mathcal{L}_{\xi\mathrm{\text{-}robust}}$ reduces the effect of exogenous states, we train a linear LAM with $\mathcal{L}_\mathrm{LAM} + \lambda \mathcal{L}_{\xi\mathrm{\text{-}robust}}$.
For an $\eta$-robust target, we used the ground-truth action $a$ and controllable changes $q_\xi$.
Figure~\ref{fig:analysis3} (a) (middle and right) shows that using action $a$ and controllable change $q_\xi$ as prediction targets $y$ (denoted Action-pred and $q$-pred, respectively) consistently improves action alignment.
In addition, Figure~\ref{fig:analysis3} (b) shows the normalized variance $\mathrm{Var}_{\xi, \xi'}(z|s,a)/\lVert z\rVert^2$ for the different latent action dimenstions $d_z$.
Adding $\mathcal{L}_{\xi\text{-}\mathrm{robust}}$ reduces this variance, indicating that the learned latent actions are more consistent across exogenous states.

\section{Practical LAM}\label{sec:practical_lam}
Our theoretical analyses rely on Assumption~\ref{assmp:exbmdp_to_linear_lam}, which is introduced for theoretical tractability.
To investigate whether our theoretical insights extend beyond the linear setting, we conduct experiments with \textit{non-additive transition dynamics} and \textit{nonlinear LAMs}.

\paragraph{Dataset.}
Motivated by \citet{linear_lam}, we construct a $4 \times 4$ grid-world dataset to evaluate practical LAM in a controlled nonlinear setting.
The top three rows contain a controllable square that moves according to one of four actions (left, right, up, down), while the bottom row contains pixel values drawn independently as $\mathrm{Bernoulli}(0.5) \cdot \sigma$ at each transition to mimic exogenous states, where $\sigma$ controls the noise intensity.
We collect transitions by following a deterministic policy that traces a snaking path through the controllable region (\textcolor[HTML]{845EC2}{purple} arrow in Figure~\ref{fig:practical_lam}(a)), so that the action is fully determined by the current observation.
This induces a strong correlation between $o$ and $a$.
Furthermore, the pixel-level multiplicative exogenous noise breaks Assumption \ref{assmp:exbmdp_to_linear_lam} on additive transition dynamics.
We sweep $\sigma$ to cover different noise regimes.

\begin{figure}[!t]
    \centering
    \includegraphics[width=\textwidth]{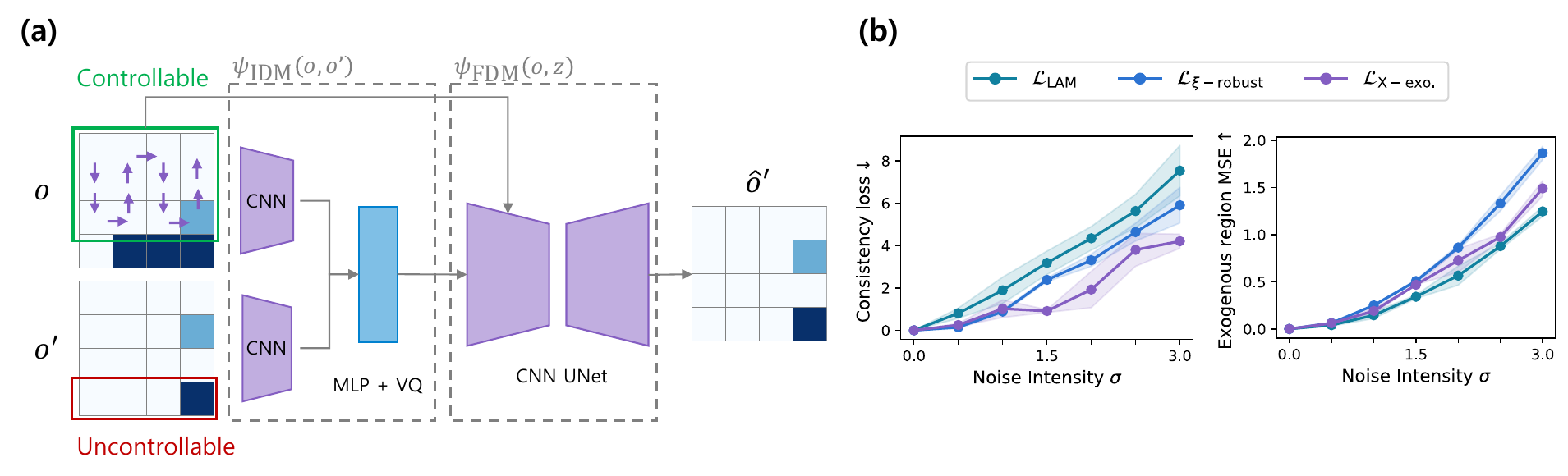} 
    \caption{\textbf{Practical LAM.}
    \textbf{(a)} Architecture of practical LAM. $\psi_\mathrm{IDM}$ is implemented with CNN and MLP layers with vector quantization (VQ), while $\psi_\mathrm{FDM}$ uses a UNet architecture with CNN layers.
    \textbf{(b)} Result of practical LAM. \textit{(Left)}: Consistency loss; lower is better. \textit{(Right)}: Exogenous region MSE; higher is better.
    Shaded regions show standard deviation across 3 random seeds.
    }
    \label{fig:practical_lam}
\end{figure}

\paragraph{Model.}
Following common practice in latent action learning~\cite{lapa, lapo}, we implement $\psi_\mathrm{IDM}$ and $\psi_\mathrm{FDM}$ as a VQ-VAE, where $\psi_\mathrm{IDM}$ consists of CNN and MLP layers followed by vector quantization, and $\psi_\mathrm{FDM}$ is a CNN-based UNet (Figure~\ref{fig:practical_lam}(a)).
We provide implementation and training details of $\psi_{\mathrm{IDM}}$ and $\psi_{\mathrm{FDM}}$ in Appendix \ref{sec:sim_detail}.
We compare a LAM trained with $\mathcal{L}_\mathrm{LAM}$ against LAMs trained with $\mathcal{L}_\mathrm{LAM} + \lambda\mathcal{L}_\mathrm{X\text{-}exo}$ and $\mathcal{L}_\mathrm{LAM} + \lambda\mathcal{L}_{\xi\text{-}\mathrm{robust}}$.
For $\mathcal{L}_\mathrm{X\text{-}exo}$, we sample paired observations under two distinct exogenous states at the same noise intensity.
For $\mathcal{L}_{\xi\text{-}\mathrm{robust}}$, we use the ground-truth action $a$ as the $\eta$-robust target.

\paragraph{Evaluation.}
To investigate whether our theoretical analysis carries over to the general latent action learning, we measure how much latent actions are corrupted by exogenous states using two metrics: (i) the pixel MSE of the uncontrollable region (\textit{Exogenous region MSE}), and (ii) the latent action MSE across exogenous states (\textit{Consistency loss}).
A higher Exogenous region MSE means the latent action does not predict the exogenous region (and thus does not encode it), while a lower Consistency loss means latent actions are consistent across exogenous states.
We report the mean and standard deviation across 3 random seeds for each configuration.

\paragraph{Result.}
Figure \ref{fig:practical_lam} (b) reports the Consistency loss (left) and Exogenous region MSE (right) as we sweep the noise intensity $\sigma$.
At $\sigma = 0$, where no exogenous noise is present, all three LAMs perform comparably.
As $\sigma$ increases, the vanilla LAM trained with $\mathcal{L}_\mathrm{LAM}$ progressively encodes exogenous information into the latent action, yielding a low Exogenous region MSE and a high Consistency loss.

In contrast, augmenting $\mathcal{L}_\mathrm{LAM}$ with either $\mathcal{L}_{\xi\text{-}\mathrm{robust}}$ or $\mathcal{L}_\mathrm{X\text{-}exo}$ keeps the Exogenous region MSE high and the Consistency loss low across all noise levels.
This indicates that auxiliary objectives in the practical LAM induce latent actions that remain consistent across exogenous variations and do not predict the exogenous region.
These results show that the auxiliary objectives improve the consistency of latent actions even when Assumption~\ref{assmp:exbmdp_to_linear_lam} does not hold.

\paragraph{Additional experiments.}
To test whether the remedies also improve action prediction and control, we conduct additional experiments with nonlinear LAMs (Appendix~\ref{sec:additional_nonlinear_lam_exp}).
On the Distracting Control Suite~\cite{dcs}, the remedies improve downstream control performance (Appendix~\ref{sec:dcs_practical_lam}).
On the real-world Bridge V2 dataset~\cite{bridge}, they reduce future leakage and improve action alignment (Appendix~\ref{sec:real_world}).
These results are consistent with our analysis of linear LAMs and show that the same remedies are also effective for practical nonlinear LAMs.

\section{Conclusion}\label{sec:conclusion}
\paragraph{Limitations and Future work.}
A limitation of our work is that our theoretical results rely on a linear LAM framework that has a gap to practical LAM.
An important direction for future work is to relax these assumptions.
Another promising direction is to extend our analysis to downstream policy learning, particularly VLA that leverage latent actions pretrained on in-the-wild videos. 

In this paper, we presented a unified theoretical analysis of LAMs under the Ex-BMDP setting.
Our analysis covers two failure modes of the standard latent action learning. 
First, the reconstruction objective induces latent actions that encode exogenous state from future observations, which is empirically known as future leakage. 
Second, the effect of exogenous noise grows when the observation space is sensitive to the exogenous state.
To address these failure modes, we analyzed two families of auxiliary objectives — cross-exogenous reconstruction and exogenous-robust target prediction — and proved that both enforce consistency of latent actions across exogenous states, thereby improving action alignment.
Experiments with linear and nonlinear LAMs support our analysis and indicate that these insights extend beyond the linear setting.

\section*{Acknowledgments}
We thank Yerin Kim for supporting the visualizations in this work.
This work is in part supported by the National Research Foundation of Korea (NRF, RS-2024-00451435(20\%), RS-2024-00413957(20\%)), Institute of Information \& communications Technology Planning \& Evaluation (IITP, RS-2025-02305453(15\%), RS-2025-02273157(15\%), RS-2025-25442149(15\%), RS-2021-II211343(15\%)) grant funded by the Ministry of Science and ICT (MSIT), Institute of New Media and Communications(INMAC), and the BK21 FOUR program of the Education, Artificial Intelligence Graduate School Program (Seoul National University), and Research Program for Future ICT Pioneers, Seoul National University in 2026.

\bibliographystyle{unsrtnat}
\bibliography{ref.bib}

\clearpage

\appendix
\section{Main Proofs}\label{sec:proof}


\subsection{Proof of Proposition~\ref{prop:context_affect}}\label{proof:context_affect}
\begin{proof}
Suppose $\theta^\star = (A^\star, B^\star, C^\star, D^\star)$ is a global minimizer of $\mathcal{L}_\mathrm{LAM}$ and that $z^\star = C^\star o + D^\star o'$ is $\xi'$-independent.
Let $\Pi$ denote the orthogonal projection onto $\mathrm{range}(B^\star)$, and let $B^{\star\dagger}$ denote the pseudo-inverse of $B^\star$, so that $B^\star B^{\star\dagger} = \Pi$.

Decomposing the residual $o' - A^\star o - B^\star z^\star$ along $\mathrm{range}(B^\star)$ and its orthogonal complement,
\begin{align}\label{eq:decomp}
\mathcal{L}_\mathrm{LAM}(\theta^\star)
= \mathbb{E}\!\left[\left\|(I - \Pi)(o' - A^\star o)\right\|^2\right]
+ \mathbb{E}\!\left[\left\|\Pi o' - \Pi A^\star o - B^\star z^\star\right\|^2\right],
\end{align}
where the cross term vanishes since $B^\star z^\star \in \mathrm{range}(B^\star)$.
Since $o = h(s, \xi)$ and $z^\star$ are both $\xi'$-independent, $\Pi A^\star o + B^\star z^\star$ is a function of $(s, \xi, a, s')$ alone—it does not depend on $\xi'$.
Hence, using $\mathbb{E}\|X - c\|^2 \geq \mathrm{Var}(X)$ for any constant $c$,
\begin{equation}\label{eq:lb}
\mathcal{L}_\mathrm{LAM}(\theta^\star)
\geq \mathbb{E}\!\left[\left\|(I - \Pi)(o' - A^\star o)\right\|^2\right]
+ \mathbb{E}_{s,\xi,a,s'}\!\left[\mathrm{Var}_{\xi'}\!\left[\Pi o^\prime \mid s, \xi, a, s'\right]\right].
\end{equation}

Now consider $\tilde{\theta} = (A^\star, B^\star, \tilde{C}, \tilde{D})$ with $\tilde{C} = -B^{\star\dagger} A^\star$ and $\tilde{D} = B^{\star\dagger}$, so that $B^\star (\tilde{C}o + \tilde{D}o^\prime) = \Pi(o' - A^\star o)$.
Substituting into the decomposition,
\begin{equation}\label{eq:loss_prime}
\mathcal{L}_\mathrm{LAM}(\tilde{\theta})
= \mathbb{E}\!\left[\left\|(I - \Pi)(o' - A^\star o)\right\|^2\right].
\end{equation}
Subtracting \eqref{eq:loss_prime} from \eqref{eq:lb},
\begin{equation}
\mathcal{L}_\mathrm{LAM}(\theta^\star) - \mathcal{L}_\mathrm{LAM}(\tilde{\theta})
\geq \mathbb{E}_{s,\xi,a,s'}\!\left[\mathrm{Var}_{\xi'}\!\left[\Pi o^\prime \mid s, \xi, a, s'\right]\right] > 0,
\end{equation}
where the strict inequality holds since the $\xi'$-dependent component of $h(s', \xi')$ has nonzero projection onto $\mathrm{range}(B^\star)$ (Section~\ref{sec:theory}).
This contradicts the global optimality of $\theta^\star$.

\end{proof}

\subsection{Proof of Proposition~\ref{prop:x_context}}\label{proof:x_context}
\begin{proof}
    \begin{align*}
        o'-\psi_\mathrm{FDM}^\mathrm{Linear}(o, \tilde{z}) &= o' - Ao -B(C\tilde{o} + D\tilde{o}')\\
        &= o+u -Ao -BC\tilde{o} -BD(\tilde{o}+\tilde{u}) \\
        &= (I-A)o -B(C+D)\tilde{o} +u - BD \tilde{u}
    \end{align*}
    Since $\mathbb{E}[ou^T] = \mathbb{E}[o\tilde{u}^T] = \mathbb{E}[\tilde{o}u^T] = \mathbb{E}[\tilde{o}\tilde{u}^T] = 0$, we have
    \begin{align*}
        \mathcal{L}_\mathrm{X\text{-}exo} &= \mathbb{E}[\lVert (I-A)o-B(C+D)\tilde{o}\rVert^2] + \mathbb{E}[\lVert u - BD\tilde{u}\rVert^2]
    \end{align*}
    The choices $A=I$ and $C=-D$ minimize $\mathcal{L}_\mathrm{X\text{-}exo}$.
    Then,
    \begin{equation*}
        \mathcal{L}_\mathrm{X\text{-}exo} = \mathbb{E}[\lVert u - BD\tilde{u} \rVert^2]
    \end{equation*}
    Let $P = BD$ and $\Sigma_{uu} = \Sigma_{\tilde{u}\tilde{u}} = I$.
    Then,
    \begin{align*}
        \mathbb{E}[\lVert u - P\tilde{u} \rVert^2]&=
        \mathrm{Tr}(\Sigma_{uu}) - 2\mathrm{Tr}(P\Sigma_{\tilde{u}u}) + \mathrm{Tr}(P\Sigma_{\tilde{u}\tilde{u}}P^T)\\
        &= d_o - 2\mathrm{Tr}(P\Sigma_{\tilde{u}u}) + \mathrm{Tr}(PP^T)\\
        &= \lVert P-\Sigma_{u\tilde{u}}\rVert_F^2 + \underbrace{d_o - \lVert \Sigma_{u\tilde{u}} \rVert^2_F}_{\text{constant}}
    \end{align*}
    Let $\Sigma_{u\tilde{u}} = \sum_{k=1}^{d_o} \sigma_k u_k v_k^T$ be the SVD.
    Recall that CCA finds pairs of directions $(a_k, b_k)$ along which two random vectors are maximally correlated, subject to each pair being uncorrelated with all previous ones:
    \begin{equation}
        (a_k^\star, b_k^\star) = \underset{\|a_k\|=\|b_k\|=1}{\arg\max}\ \mathrm{corr}(a_k^T u,\; b_k^T \tilde{u})
        \quad \text{s.t.} \quad a_k^T a_j^\star = b_k^T b_j^\star = 0, \quad j = 1,\dots, k-1.
    \end{equation}
    Since $\Sigma_{uu} = \Sigma_{\tilde{u}\tilde{u}} = I$, $\mathrm{corr}(a_k^T u, b_k^T \tilde{u}) = a_k^T \Sigma_{u\tilde{u}} b_k$
    whose solution is $(a_k, b_k) = (u_k, v_k)$ with correlation $\sigma_k$.
    By the Eckart--Young theorem, minimizing $\lVert P-\Sigma_{u\tilde{u}}\rVert^2_F$ under $\mathrm{rank}(P) \le d_z$ also yields
    \begin{equation}
        P^\star = \sum_{k=1}^{d_z} \sigma_k u_k v_k^T,
    \end{equation}
    which selects the same top-$d_z$ directions $(u_k, v_k)_{k=1}^{d_z}$.
    Therefore, under whitening, minimizing $\mathcal{L}_\mathrm{X\text{-}exo}$ recovers the top-$d_z$ canonical directions of CCA between $u$ and $\tilde{u}$.
\end{proof}

\subsection{Proof of Proposition~\ref{prop:context_robust_prediction}}\label{proof:context_robust_prediction}
\begin{proof}
For notational simplicity, let $\tilde{y} = y(\tilde{o}, \tilde{o}')$ and $\tilde{z} = \psi_\mathrm{IDM}(\tilde{o}, \tilde{o}')$.
Recall the elementary inequality
\[
\|a + b + c\|_2^2 \le 3\bigl(\|a\|_2^2 + \|b\|_2^2 + \|c\|_2^2\bigr).
\]
Since
\[
W(z - \tilde{z}) = (\tilde{y} - W\tilde{z}) - (y - Wz) + (y - \tilde{y}),
\]
it follows that
\begin{align}
\mathbb{E}\bigl[\|W(z - \tilde{z})\|_2^2\bigr]
\le
3\,\mathbb{E}\bigl[\|\tilde{y} - W\tilde{z}\|_2^2\bigr]
+ 3\,\mathbb{E}\bigl[\|y - Wz\|_2^2\bigr]
+ 3\,\mathbb{E}\bigl[\|y - \tilde{y}\|_2^2\bigr].
\end{align}
By construction,
\[
\mathcal{L}_{\xi\text{-}\mathrm{robust}}(\theta_\mathrm{IDM}, W; y)
= \mathbb{E}\bigl[\|y - Wz\|_2^2\bigr]
= \mathbb{E}\bigl[\|\tilde{y} - W\tilde{z}\|_2^2\bigr],
\]
and since $y$ is $\eta$-robust, $\mathbb{E}[\|y - \tilde{y}\|_2^2] \le \eta$. Combining these gives
\begin{equation}
\mathbb{E}\bigl[\|W(z - \tilde{z})\|_2^2\bigr]
\le
6\,\mathcal{L}_{\xi\text{-}\mathrm{robust}}(\theta_\mathrm{IDM}, W; y) + 3\eta.
\label{eq:wz_diff_bound}
\end{equation}
Now suppose $W$ has full column rank, and let $\sigma_W > 0$ denote its smallest singular value. Then for any vector $v$, $\|Wv\|_2^2 \ge \sigma_W^2 \|v\|_2^2$, which together with \eqref{eq:wz_diff_bound} yields
\begin{equation}
\mathbb{E}\bigl[\|z - \tilde{z}\|_2^2\bigr]
\le
\frac{6\,\mathcal{L}_{\xi\text{-}\mathrm{robust}}(\theta_\mathrm{IDM}, W; y) + 3\eta}{\sigma_W^2}.
\end{equation}
Hence, minimizing $\mathcal{L}_{\xi\text{-}\mathrm{robust}}$ forces the latent action $z$ to remain consistent across exogenous states, and the bound tightens as $\eta \to 0$.
\end{proof}

\section{Extended Theoretical Analysis and Discussion}
\subsection{Lower bound of leakage}\label{sec:leakage_bound}
We show that the lower bound of the future leakage $\mathrm{Var}_{\xi'}[z \mid s, \xi, a]$ is related to how large the exogenous noise is in the decoder's space.
Let $z_0 = \tilde{C} o + \tilde{D} o'$. 
Since $B^\star z_0 = \Pi\,(o' - A^\star o)$, and only $o'$ depends on $\xi'$, the decoder leakage satisfies
$$
\mathrm{Var}_{\xi'}\!\left[B^\star z_0 \mid s, \xi, a, s'\right]
= \mathrm{Var}_{\xi'}\!\left[\Pi o' \mid s, \xi, a, s'\right]
= \mathrm{Var}_{\xi'}\!\left[\Pi \varepsilon \mid s, \xi, a, s'\right].
$$
So the decoder leakage, how much of the $\xi'$ the latent action delivers to the decoder, is exactly the exogenous noise expressed in $B$'s space. 
A lower bound on the latent action leakage $\mathrm{Var}_{\xi'}[z \mid s, \xi, a]$ is not well defined on its own, since the rescaling $B \to \lambda B$, $z \to z/\lambda$ leaves the decoder leakage unchanged while shrinking the latent one arbitrarily. 
We therefore add weight decay,
$$
\mathcal{L}_\mathrm{LAM}^{\lambda}(\theta) = \mathcal{L}_\mathrm{LAM}(\theta) + \lambda \lVert \theta \rVert^2,
\qquad \lambda \in (0, 1].
$$
Let $\theta_\lambda^\star$ be its global optimizer. 
Comparing it with the closed-form decoder $\tilde{\theta}_\lambda = \big(A_\lambda^\star, B_\lambda^\star, -{B_\lambda^\star}^\dagger A_\lambda^\star, {B_\lambda^\star}^\dagger\big)$ gives $\mathcal{L}_\mathrm{LAM}^{\lambda}(\theta_\lambda^\star) \le \mathcal{L}_\mathrm{LAM}^{\lambda}(\tilde{\theta}_\lambda)$, hence
$$
e_\lambda := \mathbb{E}\!\left[
\left\lVert \Pi_\lambda (o' - A_\lambda o) - B_\lambda z_\lambda \right\rVert^2
\right]
\le \lambda \left( \lVert \tilde{\theta}_\lambda \rVert^2 - \lVert \theta_\lambda^\star \rVert^2 \right),
\qquad \Pi_\lambda = B_\lambda^\star {B_\lambda^\star}^\dagger.
$$
Applying Young's inequality $\lVert x + y \rVert^2 \le (1+\eta)\lVert x \rVert^2 + (1+1/\eta)\lVert y \rVert^2$, which holds for every $\eta>0$ and $\mathrm{Var}[X] \le \mathbb{E}[\lVert X\rVert^2]$,
$$
\mathbb{E}[\mathrm{Var}_{\xi'}[\Pi_\lambda o'|s, \xi, a, s']] \le (1+\eta)\mathbb{E}[\mathrm{Var}_{\xi'}[B_\lambda z_\lambda | s, \xi, a, s']] + \left(1+\frac{1}{\eta}\right)e_\lambda, \quad \forall \eta>0
$$
Using $e_\lambda \le \lambda(\lVert \tilde{\theta}_\lambda \rVert^2 - \lVert \theta_\lambda^\star \rVert^2)$, we have 
$$
\mathbb{E}[\mathrm{Var}_{\xi'}[B_\lambda z_\lambda|s, \xi, a, s']]\ge \frac{1}{1+\eta}\mathbb{E}[\mathrm{Var}_{\xi'}[\Pi_\lambda \varepsilon|s,\xi, a,s']] - \frac{\lambda}{\eta}(\lVert \tilde{\theta}_\lambda \rVert^2 - \lVert \theta_\lambda^\star \rVert^2) 
$$
and using $\mathrm{Var}[Bz] \le \lVert B \rVert^2_\mathrm{op}\mathrm{Var}[z]$ yields the leakage lower bound
$$
\mathbb{E}\!\left[\mathrm{Var}_{\xi'}\!\left[z_\lambda \mid s, \xi, a, s'\right]\right]
\ge \frac{1}{\lVert B \rVert^2_\mathrm{op}}\sup_{\eta>0}\left[\frac{1}{1+\eta}\mathbb{E}\!\left[\mathrm{Var}_{\xi'}\!\left[\Pi_\lambda \varepsilon \mid s, \xi, a, s'\right]\right]
- \frac{\lambda}{\eta}\left(\lVert \tilde{\theta}_\lambda \rVert^2 - \lVert \theta_\lambda^\star \rVert^2\right)\right]\ge 0.
$$
Taking the supremum is legitimate because the bound holds for every $\eta>0$, and it is non-negative since the right-hand side converges to zero as $\eta\to\infty$. 
In our linear LAM experiments, the exogenous noise is $\varepsilon = 0$ with probability $1 - p_\mathrm{switch}$ and $\varepsilon = \alpha\,(R_{\xi'} - R_\xi)\,s'$ with probability $p_\mathrm{switch}/(n_\xi - 1)$, so
$$
\mathbb{E}\!\left[\mathrm{Var}_{\xi'}\!\left[\Pi_\lambda \varepsilon \mid s, \xi, a, s'\right]\right]
= \alpha^2 \Big(
p_\mathrm{switch}\, c_1 - p_\mathrm{switch}^2\, c_2
\Big),
$$
with $c_1 = \mathbb{E}\!\left[\mathbb{E}_{\xi'\neq \xi}[\lVert \Pi_\lambda (R_{\xi'} - R_\xi)\,s' \rVert^2] \right]$ and $c_2 = \mathbb{E}\!\left[\lVert \mathbb{E}_{\xi' \neq \xi}[\Pi_\lambda (R_{\xi'} - R_\xi)\,s'] \rVert^2\right]$. Since $c_1 \ge c_2$ by Jensen and  $p\,c_1 - p^2 c_2$ increases for $p \le 1/2$, this predicts the trend in Figure~\ref{fig:analysis1}(c). 

\subsection{Proposition \ref{prop:x_context} without assumption}\label{sec:x_context_no_assmp}
In Proposition \ref{prop:x_context}, we introduce whitening and orthogonal assumption, i.e., $\Sigma_{uu} = \Sigma_{\tilde{u}\tilde{u}}$ and $\mathbb{E}[ou^T] = \mathbb{E}[o\tilde{u}^T] = 0$.
In this section, we characterize the global optimum of $\mathcal{L}_\mathrm{X-exo}$ without either assumption.
Let $v = [\tilde{o}, \tilde{u}]^T, G= A-I$, and $K=B[C+D, D]$. 
Then, cross-exogenous reconstruction objective reduces to 
\begin{equation}
   \mathcal{L}_\mathrm{X-exo} = \mathbb{E}[\lVert u-Go-Kv\rVert^2]. 
\end{equation}
For fixed $K$, the stationary condition in $G$ is $\mathbb{E}[(u-Go-Kv)o^T] = 0$, giving
\begin{equation}
    G^\star = (\Sigma_{uo} - K\Sigma_{vo})\Sigma_{oo}^{-1}.
\end{equation}
Writing $u_{\perp o} = u - \Sigma_{uo}\Sigma_{oo}^{-1}o$ and $v_{\perp o}$ similarly, substituting back gives $u - G^\star o -Kv = u_{\perp o} - Kv_{\perp o}$, so the problem is exactly
\begin{equation}
    \mathcal{L}_\mathrm{X-exo} = \mathbb{E}[\lVert u_{\perp o} - Kv_{\perp o}\rVert^2].
\end{equation}
The $Ao$ pathway thus absorbs everything linearly predictable from $o$, and the bottleneck performs reduced-rank regression (RRR) on the residuals.
As in Section \ref{proof:x_context},
\begin{equation}
    \mathbb{E}[\lVert u_{\perp o} - Kv_{\perp o}\rVert^2] = \lVert K\Sigma_{v_{\perp o}v_{\perp o}}^{1/2} - M \rVert_F^2 + \mathrm{const.}
\end{equation}
where $M = \Sigma_{u_{\perp o}v_{\perp o}}\Sigma_{v_{\perp o}v_{\perp o}}^{-1/2}$.
Since multiplying by the invertible $\Sigma_{v_{\perp o}v_{\perp o}}^{1/2}$ leaves the rank unchanged, by Eckart-Young, the optimum is the best rank-$d_z$ truncation $[\cdot]_{d_z}$ of $M$,
\begin{equation}
    K^\star = [M]_{d_z}\Sigma_{v_{\perp o}v_{\perp o}}^{-1/2}, \quad A^\star = I+(\Sigma_{uo} -K^\star\Sigma_{vo})\Sigma_{oo}^{-1}.
\end{equation}

If the orthogonal assumption is introduced, then $\Sigma_{u_{\perp o}v_{\perp o}} = [0, \ \Sigma_{u\tilde{u}}]$ and therefore $M=[0, \Sigma_{u\tilde{u}}\Sigma_{\tilde{u}\tilde{u}}^{-1/2}]$.
The optimum is $A^\star = I$, $C^\star = -D^\star$ and 
\begin{equation}
    B^\star D^\star = [\Sigma_{u\tilde{u}}\Sigma_{\tilde{u}\tilde{u}}^{-1/2}]_{d_z}\Sigma_{\tilde{u}\tilde{u}}^{-1/2}.
\end{equation}

\section{Experimental Details}\label{sec:sim_detail}
We use a single RTX 4090 GPU for both linear and practical LAM experiments.
Tables \ref{tab:hyperparam_linear_lam} and \ref{tab:hyperparam_practical_lam} show the hyperparameters of linear and practical LAM experiments, respectively.

\begin{table}[!h]
\caption{\textbf{Hyperparameters for the linear LAM experiments.} We use the same training setup across all experiments.}
\label{tab:hyperparam_linear_lam}
\centering
\begin{tabular}{l|c}
\toprule
\multicolumn{2}{c}{\textbf{Hyperparameters of linear LAM experiments}} \\ \midrule
$d_s$                              & 8                 \\
$d_a$                              & 8                 \\
$d_o$                              & 128               \\
num. traj. of train               & 8k                \\
Batch size                        & 128               \\
Steps/Probe Steps                 & 20k               \\
Learning rate                     & 1e-3       \\ \bottomrule      
\end{tabular}
\end{table}

\begin{table}[!h]
\caption{\textbf{Hyperparameters for the practical LAM}  Architecture (up) and training settings (down) for the practical LAM.}
\label{tab:hyperparam_practical_lam}
\centering
\begin{tabular}{l|c}
\toprule
\multicolumn{2}{c}{\textbf{Hyperparameters of Practical LAM}} \\ \midrule
$\psi_\mathrm{IDM}$ conv.\ channels & [128, 128, 128] \\
$\psi_\mathrm{IDM}$ MLP hidden dim. & 64 \\
$\psi_\mathrm{FDM}$ conv.\ channels & [32, 32, 32, 32] \\
Latent action dim.\ ($d_z$)         & 32 \\
Codebook size ($K$)                 & 5 \\
VQ commitment weight ($\beta$)      & 0.25 \\ \midrule
Optimizer                           & Adam \\
Learning rate                       & 3e-4 \\
Batch size                          & 128 \\
Total updates                       & 16k \\
Gradient clip norm                  & 5.0 \\
$\lambda_\mathrm{X\text{-}exo}$     & 1.0 \\
$\lambda_\mathrm{act}$              & 1.0 \\
Action label fraction               & 1\% \\ \bottomrule
\end{tabular}
\end{table}

\section{Additional Experiments}\label{sec:additional_nonlinear_lam_exp}
\subsection{Distracting Control Suite}\label{sec:dcs_practical_lam}
The grid-world dataset in Section~\ref{sec:practical_lam} does not fully reflect the complexity of real-world video data.
We adopt it primarily to validate our theoretical findings in a controlled nonlinear setting, which is difficult to isolate in real-world datasets.
To partially close this gap, we additionally conduct practical LAM experiments on the Distracting Control Suite~\cite{dcs}.

\paragraph{Dataset.}
We collect offline datasets from the DeepMind Control Suite~\cite{control_suite}, focusing on the \texttt{cheetah-run} environment.
We use the pretrained PPO agent to collect 500 expert demonstrations.
Note that the behavior of the PPO expert induces a strong correlation between $o$ and $a$.
To mimic in-the-wild videos, we further inject exogenous noise by perturbing the color and viewpoint of each observation.
Specifically, we use $(\Delta_{\mathrm{color}}, \Delta_{\mathrm{cam}}) = (0.5, 0.05)$, where $\Delta_{\mathrm{color}}$ controls the maximum per-channel color perturbation and $\Delta_{\mathrm{cam}}$ scales the magnitude of camera pose and zoom randomization.
Both color and camera parameters drift continuously across timesteps.

\paragraph{Model.}
We follow the architecture of LAPO~\cite{lapo} to implement the practical latent action model.
Both $\psi_\mathrm{IDM}$ and $\psi_\mathrm{FDM}$ are implemented with nonlinear layers consisting of several convolution layers with GELU and Tanh activations.
We remove the FSQ bottleneck used in the original implementation since it hindered overall performance in our experiments, and instead employ a strong information bottleneck ($d_z = 8$).
To investigate the quality of latent actions beyond the linear probing, we additionally train a latent action policy $\pi(z|o)$ and an action decoder $\pi_{\mathrm{Dec}}(a|z)$ with a simple MSE objective.
This protocol mirrors common scenarios in which latent actions are used to supervise VLA training.
We implement $\pi(z|o)$ with convolution layers and a linear head with ReLU6 activation, while $\pi_\mathrm{Dec}(a|z)$ is a 3-layer MLP.
We evaluate the learned policy and report the normalized return.
As an upper bound, we also train a behavior-cloning policy $\pi_\mathrm{BC}(a|o)$ directly on the labeled dataset.

\begin{figure}[!t]
    \centering
    \includegraphics[width=\textwidth]{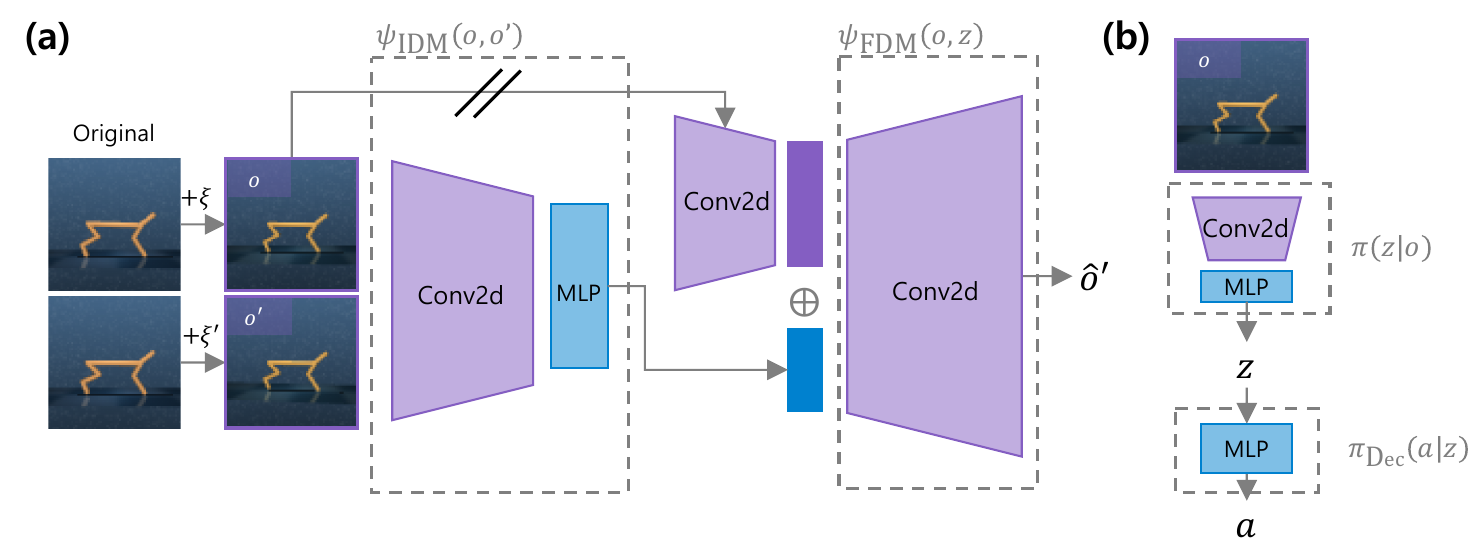}
    \caption{\textbf{Architecture of practical LAM on DCS.} (a) We first collect clean demonstrations from a PPO agent and inject exogenous noise into each observation. Both $\psi_\mathrm{IDM}$ and $\psi_\mathrm{FDM}$ are implemented with several convolution layers.
    (b) The latent action policy $\pi(z|o)$ predicts $z$ from $o$ using convolution layers, and $\pi_\mathrm{Dec}(a|z)$ decodes the prediction of the latent action policy into a ground-truth action.}
    \label{fig:dcs_practical}
\end{figure}

\paragraph{Training objectives.}
We compare the vanilla LAM trained with $\mathcal{L}_\mathrm{LAM}$ against the auxiliary objectives studied in Section~\ref{sec:theory}: cross-exogenous reconstruction $\mathcal{L}_\mathrm{X\text{-}exo}$, and exogenous-robust target prediction $\mathcal{L}_{\xi\text{-}\mathrm{robust}}$ instantiated with action labels and optical flow as $\eta$-robust targets.
In addition, motivated by \citet{dynamo}, we consider a variant where a vision encoder $f(\cdot)$ is jointly trained with the LAM.
For each image $I$, $f(\cdot)$ extracts a representation $o = f(I)$ that the LAM uses as its observation, and for the next observation $o' = f(I')$, $f(\cdot)$ is treated as stop-gradient and updated via an exponential moving average (EMA) of the online encoder.
\citet{dynamo} report that this representation improves downstream manipulation performance, suggesting it captures more endogenous information.
Within our framework, this corresponds to $f(\cdot)$ providing a representation with low $\xi$-sensitivity $\delta_h$, which is predicted to reduce exogenous noise energy and improve LAM training (Equation~\eqref{eqn:noise_energy_sensitivity}).

\paragraph{Results.}
Results are reported in Table~\ref{tab:res_practical_lam}.
The EMA-encoder variant ($+f(\cdot)$) substantially outperforms the vanilla LAM. 
In our framework, this can be interpreted as $f(\cdot)$ learning a representation with low $\xi$-sensitivity $\delta_h$, which is predicted to improve LAM training (Equation~\eqref{eqn:noise_energy_sensitivity}).
The auxiliary training objectives ($\mathcal{L}_\mathrm{X\text{-}exo}$, action prediction, and flow prediction) also yield consistent improvements over the vanilla LAM, in agreement with our theoretical predictions.

\begin{table}[!h]
\small
\centering
\caption{\textbf{Normalized return across auxiliary objectives on \texttt{cheetah-run}.} Results are reported as $\mathrm{mean}_{\pm \mathrm{std}}$ over 100 episodes.}\label{tab:res_practical_lam}
\begin{tabular}{l|cc|cccc} \toprule
Task  & $\pi_\mathrm{BC}(a|o)$ & $\mathcal{L}_\mathrm{LAM}$ &  $+f(\cdot)$ & $+\mathcal{L}_\mathrm{X\text{-}exo}$ & +Action-pred & +Flow-pred \\  \midrule
Normalized Return & $0.18_{\pm0.12}$ & $0.05_{\pm 0.04}$ & $0.15_{\pm 0.06}$ & $0.16_{\pm 0.07}$ & $0.15_{\pm 0.11}$ & $0.18_{\pm 0.10}$ \\ 
One-sided $p$ & -- & -- & <0.001 & <0.001 & <0.001 & <0.001 \\
\bottomrule
\end{tabular}
\end{table}

\subsection{Validation in Real-world dataset}\label{sec:real_world}
\begin{table}[!h]
\caption{Result of latent action learning across the methods in Bridge V2 dataset. Results are shown in $\mathrm{mean}_{\pm \mathrm{std}}$ across 5 training seeds.} 
\label{tab:real_world_result}
\small
\begin{tabular}{lc|cccc} \toprule
Method & Trans. NMSE & $\mathrm{PSNR}(\hat{o}', o')\ \uparrow$           & $\mathrm{PSNR}(\hat{o}', \tilde{o}')\ \downarrow$          & $\Delta\mathrm{PSNR}$          & 95\% CI              \\ \midrule
 $\mathcal{L}_\mathrm{LAM}$     & $0.419_{\pm 0.007}$   & $20.46_{\pm 0.22}$ & $21.15_{\pm 0.23}$ & $+0.69_{\pm 0.44}$ & $[+0.150, +1.230]$ \\
  $+\lambda\mathcal{L}_{\xi\text{-robust}}$     & $0.120_{\pm 0.004}$   & $22.76_{\pm 0.11}$ & $18.81_{\pm 0.05}$ & $-3.95_{\pm 0.12}$ & $[-4.094, -3.796]$ \\
  $+\mathcal{L}_\mathrm{X-exo}$     & $0.359_{\pm 0.019}$ & $22.27_{\pm 0.08}$ & $19.09_{\pm 0.10}$ & $-3.18_{\pm 0.13}$ & $[-3.343, -3.008]$ \\
  $+f_\mathrm{dino}(\cdot)$     & $0.315_{\pm 0.014}$   & $22.56_{\pm 0.09}$ & $19.20_{\pm 0.08}$ & $-3.36_{\pm 0.08}$ & $[-3.459, -3.262]$\\ \bottomrule
\end{tabular}
\end{table}

\paragraph{Nonlinear LAM}
To validate our theoretical results on real-world data, we train LAPA's LAM from scartch on the Bridge V2 dataset. 
We train LAMs with $\mathcal{L}_\mathrm{LAM}$, $\mathcal{L}_\mathrm{LAM} + \lambda \mathcal{L}_{\xi\text{-robust}}$(action for $\eta$-robust target), and $\mathcal{L}_\mathrm{LAM} + \mathcal{L}_\mathrm{X-exo}$, and evaluate two quantities.
\begin{enumerate}
    \item Translation NMSE: how well the latent action recovers the ground-truth end-effector translation with a single linear layer.
    We exclude the rotation and gripper axes because no model predicts them better than the mean predictor, whose NMSE is approximately 1, at this scale, where every LAM is trained from scratch on a small dataset.
    \item $\Delta\mathrm{PSNR}$: the future leakage diagnostic of Figure \ref{fig:analysis1}, evaluated on the $\xi^\prime$\texttt{-swap}. 
    We compute $\mathrm{PSNR}(\hat{o}', \tilde{o}') - \mathrm{PSNR}(\hat{o}', o')$, where $\hat{o}'$ is the output of $\xi^\prime$-\texttt{swap}.
    A positive value means the latent action reproduces the swapped future, i.e., it leaks the future exogenous state.
\end{enumerate}

We additionally include a frozen DINOv2 encoder $f_\mathrm{dino}(\cdot)$ as a low-sensitivity backbone, which tests the effect of encoder-sensitivity in reliable latent action learning.
The training procedure is identical across all objectives, with the same architecture and hyperparameters, so that the objective (and the encoder initialization) is the only variable. 
Each model is trained for 10k steps and we report the early-stopped checkpoint selected by validation reconstruction loss. 

Table~\ref{tab:real_world_result} shows that the vanilla LAM leaks the future exogenous state, while the LAMs trained with the three remedies do not.
This is significant under a one-sided $t$-test with $n=5$, giving $p = 0.012$ for $\mathcal{L}_\mathrm{LAM}$ and $p<0.001$ for remaining methods.
Also, none of the four 95\% confidence intervals contains zero.

\clearpage

\end{document}